\documentclass[11pt]{article}
\usepackage{times}
\usepackage{enumitem}
\usepackage{fullpage}

\usepackage{jmlrutils}
\usepackage{amsmath,amsfonts,amssymb}
\usepackage{xcolor}
\usepackage{algorithm, algorithmic}
\usepackage{xspace}
\usepackage{natbib}
\newtheorem{assumption}{Assumption}

\newcommand{\cX}{\ensuremath{\mathcal{X}}}

\newcommand{\cF}{\ensuremath{\mathcal{F}}}

\newcommand{\R}{\ensuremath{\mathbb{R}}}

\renewcommand{\P}{\ensuremath{\mathbb{P}}}

\newcommand{\rE}{{\mathbb E}}

\DeclareMathOperator*{\argmax}{\text{argmax}}
\DeclareMathOperator*{\argmin}{\text{argmin}}

\DeclareMathOperator*{\arginf}{\text{arginf}}

\newcommand{\cP}{\ensuremath{\mathcal{P}}}
\newcommand{\Rhat}{\ensuremath{\widehat{R}}}
\newcommand{\fhat}{\ensuremath{\widehat{f}}}
\newcommand{\order}{\mathcal{O}}

\newcommand{\dref}{\ensuremath{P_0}}
\newcommand{\cW}{\ensuremath{\mathcal{W}}}

\newcommand{\cZ}{\ensuremath{\mathcal{Z}}}

\newcommand{\sigmah}{\ensuremath{\widehat{\sigma}}}
\newcommand{\betah}{\ensuremath{\widehat{\beta}}}

\newcommand{\dro}{\ensuremath{\text{DRO}}\xspace}

\newcommand{\dmu}{\ensuremath{\Delta_{\mu}}}
\newcommand{\ds}{\ensuremath{\Delta_S}}

\newcommand{\regret}{{\mathrm{Regret}}}

\newcommand{\event}{\ensuremath{\mathcal{E}}}

\usepackage[breaklinks=true]{hyperref}
\usepackage{breakcites}

\newcommand{\alg}{\ensuremath{\text{MRO}}\xspace}
\newcommand{\alglong}{Minimax Regret Optimization\xspace}
\newcommand{\algscale}{\ensuremath{\text{SMRO}}\xspace}

\usepackage{nicefrac}

\title{\alglong for Robust Machine Learning under Distribution Shift}
\author{Alekh Agarwal \\alekhagarwal@google.com\\Google Research \and Tong Zhang \\tozhang@google.com\\Google Research \& HKUST}

\begin{document}
\date{}
\maketitle

\begin{abstract}
    In this paper, we consider learning scenarios where the learned model is evaluated under an unknown test distribution which potentially differs from the training distribution (i.e. distribution shift). The learner has access to a family of weight functions such that the test distribution is a reweighting of the training distribution under one of these functions, a setting typically studied under the name of Distributionally Robust Optimization (\dro). 
    We consider the problem of deriving regret bounds in the classical learning theory setting, and require that the resulting regret bounds hold uniformly for all potential test distributions. We show that the \dro formulation does not guarantee uniformly small regret under distribution shift. We instead propose an alternative method called \alglong (\alg), and show that under suitable conditions this method achieves \emph{uniformly low regret} across all test distributions.  
    %Observing that \dro can be notoriously prone to overly emphasizing distributions with higher noise levels, we instead propose an alternative approach we call \alglong (\alg), which finds a function that has \emph{uniformly low regret} across all the weight functions. 
    %We illustrate the benefits of \alg over \dro through illustrative examples, and analyze the quality of the learned model under different settings. 
    We also adapt our technique to have stronger guarantees when the test distributions are heterogeneous in their similarity to the training data. %well-specified settings, where vanilla ERM is already a strong baseline and direct importance weighting can worsen the finite sample bounds. 
    Given the widespead optimization of worst case risks in current approaches to robust machine learning, we believe that \alg can be a strong alternative to address distribution shift scenarios.
\end{abstract}

\section{Introduction}

Learning good models for scenarios where the evaluation happens under a test distribution that differs from the training distribution has become an increasingly important problem in machine learning and statistics. It is particularly motivated by the widespread practical deployment of machine learning models, and the observation that model performance deteriorates when the test distribution differs from the training distribution. Some motivating scenarios range from class imbalance between training and test data~\citep{galar2011review,japkowicz2000class} to algorithmic fairness~\citep{dwork2012fairness,barocas2016big}. Consequently, several recent formalisms have been proposed to develop algorithms for such settings, such as adversarial training~\citep{goodfellow2014explaining,madry2017towards,wang2019convergence}, Invariant Risk Minimization~\citep{arjovsky2019invariant,ahuja2020invariant,rosenfeld2020risks} and Distributionally Robust Optimization~\citep{duchi2021learning,namkoong2016stochastic,staib2019distributionally,kuhn2019wasserstein}. While the specifics differ across these formulations, at a high-level they all minimize the worst-case risk of the learned model across some family of test distributions, and spend considerable effort on finding natural classes of test distributions and designing efficient algorithms to optimize the worst-case objectives. 

In this paper, we consider this problem from a learning theory perspective, by asking if it is possible to obtain regret bounds that hold uniformly across all test distributions. Recall that in supervised learning, we are given a function class $\cF$ and a loss function $\ell(z,f(z))$, where $z$ denotes a sample. For a test distribution $P$, we are interested in learning a prediction function $f \in \cF$ using the training data, so that the risk of $f$ under $P$: 
$R_{P}(f)=\rE_{z \sim P} \ell(z,f(z))$,
is small. In classical learning theory, an estimator $\hat{f}$ is learned from the training data, and we are interested in bounding the regret (or excess risk) of $f$ as follows
\[
\regret_P(\hat{f})= R_{P}(\hat{f}) - \inf_{f \in \cF} R_P(f) = R_{P}(\hat{f}) - R_P(f_P),
\]
where $f_P$ minimizes $R_P(f)$ over $\cF$.
If the training data $(z_1,\ldots,z_n)$ are drawn from the same distribution as the test distribution $P$, then standard results imply that the empirical risk minimizaiton (ERM) method, which finds an $\hat{f}$ with a small training error, achieves small regret. A natural question we would like to ask in this paper is whether there exists an estimator that allows us to achieve small regret with distribution shift? More precisely, suppose we are given a family of distributions $\cP$ with $\dref \in \cP$. Given access to $n$ samples $(z_1,\ldots,z_n)$ from $\dref$, we would like to find an estimator $\hat{f}$ from the training data so that the uniform regret
\begin{equation}
\sup_{P \in \cP} \regret_P(\hat{f})= \sup_{P \in \cP}  \left[ R_{P}(\hat{f}) - \inf_{f \in \cF} R_P(f) \right] 
\label{eq:uniform-regret}
\end{equation}
is small.  Here the model class $\cF$ can be statistically misspecified in that it may not contain the optimal model $f^\star$ that attains the Bayes risk in a pointwise manner. 

We note that this objective is different from the objective of Distributionally Robust Optimization (henceforth \dro), which directly minimizes the risk across all test distributions as follows:
\begin{equation}
    f_{\dro} = \arginf_{f\in\cF} \sup_{P\in\cP}  R_P(f) .
    \label{eq:dro-obj}
\end{equation}
We observe that the objective~\eqref{eq:uniform-regret} is less sensitive than \dro to 
heterogeneity in the amount of noise in test distributions. Moreover, when the model is well-specified, our criterion directly measures the closeness of the estimator $\hat{f}$ and the optimal model $f^\star$, which is often desirable in applications. 
\iffalse
Focusing on the concrete setting studied in Distributionally Robust Optimization (henceforth \dro), we assume that the learner has access to $n$ samples $(z_1,\ldots,z_n)$ from a training distribution $\dref$. Given a class of functions $\cF$, the learning goal is to find a function $f\in\cF$ which makes good predictions under test data drawn from test distributions $P\in\cP$, where the quality of predictions is measured using some loss function $\ell(z,f(z))$ that we seek to minimize. The \dro literature formalizes this learning goal as trying to perform as well as the benchmark 
\begin{equation}
    f_{\dro} = \arginf_{f\in\cF} \sup_{P\in\cP}  \left\{R_P(f) = \rE_{z\sim P} \ell(z,f(z))\right\}.
    \label{eq:dro-obj}
\end{equation}
In Example~\ref{ex:dro-noise} in Section~\ref{sec:setting}, we show that this objective has undesirable properties even when $\cP$ consists of just two distributions that have widely different noise levels. Indeed if every $f\in\cF$ has a large value of $R_P(f)$ under some distribution $P\in\cP$, then \dro degenerates to minimizing the risk on that distribution only, potentially significantly underperforming on other distributions. This observation suggests that the \dro objective~\eqref{eq:dro-obj} does an inadequate job of formalizing our original goal of doing as well as possible across distributions in $\cP$. 
\fi
We call the criterion to minimize regret uniformly across test distributaions  \alglong (\alg),  and its population formulation seeks to minimize the worst-case regret~\eqref{eq:uniform-regret}:% defined above: 
\begin{equation}
    f_{\alg} = \arginf_{f\in\cF} \sup_{P\in\cP}  \regret_P(f) .
    \label{eq:minimax-pop}
\end{equation}
Compared to \dro, \alg evaluates the regret of a candidate model $f$ on each distribution $P\in\cP$ as opposed to the raw risk. As we show in the following sections, regret is comparable across distribution in a more robust manner than the risk, since the subtraction of the minimum risk for each distribution takes away the variance due to noise. We note that a similar approach to remove the residual variance to deal with heterogenous noise arises in the offline Reinforcement Learning setting~\citep{antos2008learning}, which we comment more on in the discussion of related work.
 Similar to ERM, which replaces the regret minimization problem over a single test distribution by its empirical minimization counterpart, we consider an empirical minimization counterpart of \alg, and analyze its generalization performance. 

\paragraph{Our contributions.} With this context, our paper makes the following contributions.

\begin{enumerate}[leftmargin=*,itemsep=0pt]
    \item The objective~\eqref{eq:minimax-pop}, which gives an alternative to \dro, alleviates its shortcomings related to noise overfitting, and provides a robust learning formulation for problems with distribution shift.
    \item We show the empirical couterpart to \eqref{eq:minimax-pop} can be analyzed using classical tools from learning theory.
    When the loss function is bounded and Lipschitz continuous, we show that the regret of our learned model $\hat f_{\alg}$ on each distribution $P$ can be bounded by that of $f_{\alg}$ along with a deviation term that goes down as $\order(1/\sqrt{n})$. Concretely, we show that if $dP(z)/d\dref(z)\leq B$ for all $P\in\cP$, then with high probability, 
    \begin{equation*}
        \sup_{P\in\cP} \regret_P(\fhat_{\alg})  \leq  \inf_{f\in\cF}\sup_{P\in\cP} \regret_P(f)  + \order\left(\frac{B}{\sqrt{n}}\right),
    \end{equation*}
    where we omit the dependence on the complexities of $\cF$ and $\cP$ as well as the failure probability and use a simplified assumption on $dP/dQ$ than in our formal result of Theorem~\ref{thm:slow} for ease of presentation. That is, $\fhat_{\alg}$ has a bounded regret on each domain, so long as there is at least one function with a small regret on each domain. 
    %In contrast, \dro solution $\fhat_{\dro}$ only has a bound on the worst case risk across domains~\citep{duchi2021learning}, which can lead to substantially inferior guarantees on regret as we illustrate in Sections~\ref{sec:setting} and~\ref{sec:theory}.
    For squared loss, if the class $\cF$ is convex, we show that our rates improve to $\order(B/n)$ (Theorem~\ref{thm:fast}). The result follows naturally as empirical regret concentrates to population regret at a fast rate in this setting. We also show that the worst-case risk of $\fhat_{\dro}$ cannot approach the worst case risk of $f_{\dro}$ at a rate faster than $1/\sqrt{n}$ (Proposition~\ref{prop:dro-slow}), which demonstrates a theoretical benefit of our selection criterion.
    
    \item We present \algscale, an adaptation of our basic technique which further rescales the regret of each distribution differently. We show that \algscale retains the same worst-case guarantees as \alg, but markedly improves upon it when the distributions $P\in \cP$ satisfy an \emph{alignment} condition, which includes all well-specified settings (Theorems~\ref{thm:slow-scale} and~\ref{thm:fast-scale}). The rescaling parameters are fully estimated from data. \iffalse As an intuitive example, if each distribution $P\in\cP$ is a covariate shift of $\dref$ in the well-specified linear regression setting with optimal weights $\beta^\star$ and $\dref \in \cP$, then \algscale learns regression weights satisfying $\|\betah_{\algscale} - \beta^\star\|_2^2 = \order(d/(\lambda n))$ up to a term depending on the complexity of $\cP$, with no dependency on $B$! This almost matches the guarantee of ERM, which in the well-specified setting, is a asymptotically efficient for estimating $\beta^\star$ under proper conditions. \fi
    
    \item Algorithmically, we show that our method can be implemented using access to a (weighted) ERM oracle for the function class $\cF$, using the standard machinery of solving minimax problems as a two player game involving a no-regret learner and a best response player. The computational complexity scales linearly in the size of $\cP$ and as $n^2$ in the number of samples. 
\end{enumerate}

\iffalse
\begin{table}[!tb]
\begin{tabular}{|c|c|c|c|c|}
&\dro&ERM&\alg&\algscale\\
Bdd, Lipschitz loss & $\sup_{Q\in\cP} R_Q(f_\dro) - R_P(f_P) + \order(B/\sqrt{n})$ & $BR_{\dref}(f_{\dref})-R_P(f_P) + \order(B/\sqrt{n})$ & $\sup_{P\in\cP} R_P(f_{\alg}) + \order(B/\sqrt{n})$ & $\sup_{P\in\cP} R_P(f_{\alg}) + \order(B/\sqrt{n})$
\end{tabular}
\label{tbl:results}
\end{table}
\fi
We conclude this section with a discussion of related work.

\subsection{Related Work}

Distributional mismatch between training and test data has been studied in many settings (see e.g.~\citep{quinonero2009dataset}) and under several notions of mismatch. A common setting is that of covariate shift, where only the distribution of labels can differ between training and test~\citep{shimodaira2000improving,huang2006correcting,bickel2007discriminative}. Others study changes in the proportions of the label or other discrete attributes between training and test~\citep{dwork2012fairness,xu2020class}. More general shifts are considered in domain adaptation~\citep{mansour2009domain,ben2010theory,patel2015visual} and transfer learning~\citep{pan2009survey,tan2018survey}. 

A recent line of work on Invariant Risk Minimization (IRM)~\citep{arjovsky2019invariant} considers particularly broad generalization to unseen domains from which no examples are available at the training time,  motivated by preventing the learning of spurious features in ERM based methods. The idea is to find an estimator which is invariant across all potential test distributions. Although most empirical works on IRM use gradient-penalty based formulations for algorithmic considerations, one interpretation of IRM leads to a minimax formulation similar to the \alg objective~\eqref{eq:minimax-pop}:
%. Concretely, IRM can be formalized as minimizing the training loss, subject to a constraint on the minimax regret:
\begin{equation}
    f_{\text{IRM}} = \arginf_{f\in\cF} R_{\dref}(f) \quad \mbox{s.t.} \sup_{P \in \cP} R_P(f) - R_P(f_P) \leq \epsilon,
    \label{eq:irm}
\end{equation}
where the family $\cP$ consists of all individual domains in the training data and $\dref$ is the entire training data pooled across domains. However, it is not obvious how to meaningfully analyze the distributional robustness of IRM method \eqref{eq:irm} in the classical learning theory setting, 
 and some subsequent works works~\citep{ahuja2020empirical,kamath2021does} formalize IRM as minimizing a \dro style worst-case risk objective, instead of regret as in the original paper. In contrast, our \alg formulation is motivated by extending regret-driven reasoning to settings with distribution shift. 

Distributionally robust optimization, which forms a starting point for this work, has a long history in the optimization literature (see e.g.~\citep{ben2009robust,shapiro2017distributionally}) and has gained recent prominence in the machine learning literature~\citep{namkoong2016stochastic,duchi2021learning,duchi2021statistics,staib2019distributionally,kuhn2019wasserstein,zhu2020kernel}. DRO has been applied with mixed success in language modeling~\citep{oren2019distributionally}, correcting class imbalance~\citep{xu2020class} and group fairness~\citep{hashimoto2018fairness}. For most of the theoretical works, the emphasis is on efficient optimization of the worst-case risk objective. While \citet{duchi2021learning} provide sharp upper and lower bounds on the risk of the \dro estimator, they do not examine regret, which is the central focus of our study.

We note that the raw risk used in \dro is sensitive to heterogeneous noise, in that larger noise leads to larger risk. Such sensitivity can be undesirable for some applications. 
Challenges in learning across scenarios with heterogeneous noise levels has been previously studied in supervised learning~\citep{crammer2005learning}, and also arise routinely in reinforcement learning (RL). In the setting of offline RL, a class of techniques is designed to minimize the Bellman error criterion~\citep{bertsekas1996neuro,munos2008finite}, which has known challenges in direct unbiased estimation (see e.g. Example 11.4 in~\citet[][]{sutton1998reinforcement}). To counter this, \citet{antos2008learning} suggest removing the residual variance, which is akin to considering regret instead of risk in our objective and yields similar minimax objectives as this work. 

On the algorithmic side, we build on the use of regret minimization strategies for solving two player zero-sum games, pioneered by~\citet{freund1996game} to allow for general distribution families $\cP$ as opposed to relying on closed-form maximizations for specific families, as performed in many prior works~\citep{namkoong2016stochastic,staib2019distributionally,kuhn2019wasserstein}. Similar approaches have been studied in the presence of adversarial corruption in~\citet{feige2015learning} and in the context of algorithmic fairness for regression in~\citet{agarwal2019fair}.

\section{Setting}
\label{sec:setting}
We consider a class of learning problems where the learner is given a dataset sampled from a distribution $\dref$ of the form $z_1,\ldots,z_n$ with $z_i\in\cZ$. We also have a class $\cP$ of distributions such that we want to do well relative to this entire class of distributions, even though we only have access to samples from $\dref$. In \alglong, we formalize the property of doing uniformly well for all $P\in \cP$ through the objective~\eqref{eq:minimax-pop}. As mentioned in the introduction, this is closely related to the \dro objective~\eqref{eq:dro-obj}. 
The following result shows that small risk does not lead to small regret, which means \dro does not solve  the \alg objective~\eqref{eq:minimax-pop}. The construction also shows that  under heterogenous noise,
\dro tends to focus on a distribution $P\in\cP$ with large noise level, which can be undesirable for many applications.

\begin{proposition}[\dro is sensitive to noise heterogeneity]
There exists a family of two distributions $\cP=\{P_1,P_2\}$ with heterogeneous noise levels,
and a function class $\cF=\{f_1,f_2\}$, so that \\$\sup_{P\in\cP}\regret_P(f_{\dro}) = 0.21$, while $\sup_{P\in\cP}\regret_P(f_{\alg}) = 0.03$. 
\label{prop:dro-noise}
\end{proposition}
\begin{proof}
Suppose $\cP$ consists of two distributions over $z\in[0,1]$. The first distribution $P_1$ is degenerate with $P_1(0.1) = 1$. The second distribution $P_2 = \text{Ber}(0.5)$. The goal is to estimate the mean of the distributions using least squares loss. In this example, $P_1$ has zero-noise and $P_2$ has a relatively high noise.
We consider a function class $\cF$ consisting of two estimates of the mean, $f_1 = 0.3$ and $f_2 = 0.6$ and $\ell(z,f) = (f - z)^2$. Let us abbreviate $R_{P_i}(f) = R_i(f)$ for $i\in\{1,2\}$.
Note that the true mean of $P_1$ is $0.1$, and that of $P_2$ is $0.5$. Thus $f_1$ is a better uniform estimate of the means. 

It is easy to verify that $R_1(f_1) = 0.04$, $R_1(f_2) = 0.25$, $R_2(f_1) = 0.29$, and $R_2(f_2) = 0.26$. Then the \dro objective~\eqref{eq:dro-obj} evaluates to the optimal value of 0.26, with $f_2$ attaining this value. On the other hand, the \alg objective~\eqref{eq:minimax-pop} is equal to $0.03$ and this value is attained for $f_1$, while $f_2$ incurs a much higher objective of $\sup_P R_P(f_2) - R_P(f_P) = 0.21$. As a result, we have
%\begin{align*}
    $f_{\dro} = f_2$ and $f_{\alg} = f_1$. 
%\end{align*}
Comparing regrets lets \alg remove the underlying noise level of $0.25$ for $P_2$.
\end{proof}

Note that while the worst-case regret of \alg is always better than that of \dro, this does not imply that the function $f_{\alg}$ is always preferable under each distribution $P\in\cP$ to $f_{\dro}$. Next we give one such example where $f_{\dro}$ achieves an arguably better trade-off in balanced performance across distributions. This situation can happen in real applications when functions achieving minimum risks vary greatly across distributions, potentially due to overfitting. The example shows that 
 both \alg and \dro may have advantages or disadvantages under different conditions. 

\begin{example}[Heterogeneous regrets across domains]
Suppose the target class $\cP$ consists of two distributions over $z\in[0,1]$ and our function class $\cF$ consists of 3 functions $\{f_1, f_2, f_3\}$. Let us assume that the risks of these functions under the two distributions are given by:
\begin{align*}
    R_1(f_1) = 0,~~R_1(f_2) = 0.5,~~R_1(f_3) = 0.5+\epsilon,~~R_2(f_1) = 1, R_2(f_2) = 0.9,~~\mbox{and}~~R_2(f_3) = 0.4.
\end{align*}
Then $f_1$ is easily disregarded under both \alg and \dro, since it has poor performance under $P_2$. For both $f_2$ and $f_3$, their regrets on $P_1$ are larger than $P_2$, so \alg selects the better function for $P_1$, that is, $f_{\alg} = f_2$. \dro, on the other hand, prefers $f_3$ as it has a smaller worst-case risk, which is arguably the right choice in this scenario. 
\label{ex:mro-bad}
\end{example}
More generally, one can always find a distribution $P\in\cP$ under which the regret of $f_{\dro}$ is smaller than that of $f_{\alg}$ and vice versa, but the ordering in the worst-case is clear by definition. In the context of these observations, we next discuss how we might estimate $f_{\alg}$ from samples, for which it is useful to rewrite the objective~\eqref{eq:minimax-pop} in an equivalent weight-based formulation we discuss next. 
\paragraph{A weight-based reformulation.}
Having shown the potential benefits of our population objective, we now consider how to optimize it using samples from $\dref$. Notice that the objective~\eqref{eq:minimax-pop} does not depend on $\dref$ explicitly, which makes it unclear how we should approach the problem given our dataset. We address this issue by adopting an equivalent reweighting based formulation as typically done in \dro. Concretely, let us assume that $P$ is absolutely continous with respect to $\dref$ for all $P\in\cP$, so that there exists
a weighting functions $w~:~\cZ\to\R_+$ such that
$d P(z) = w(z) d\dref(z)$, where $\rE_{\dref}[w] = 1$.
We can equivalently rewrite the objective~\eqref{eq:minimax-pop} as

\begin{align}
    f_{\alg} &:= \arginf_{f \in \cF}\sup_{w\in\cW}\left\{ \regret_w(f)
:=       R_w(f) - \inf_{f'\in \cF} R_w(f') = \ R_w(f) - R_w(f_w)\right\},\label{eq:minimax-pop-w}
\end{align}
where $\cW = \{w~:~w(z) = dP(z)/d\dref(z) ~~\text{for}~~P\in\cP\}$ and $R_w(f) = \rE_{z\sim P_0}[w(z)\ell(z,f(z))]$. It is also straightforward to define an empirical counterpart for this objective. Given $n$ samples $z_1,\ldots,z_n \sim \dref$, we can define the (weighted) empirical risk and its corresponding minimizer as $\Rhat_w(f) = \frac{1}{n}\sum_{i=1}^n w(z_i) \ell(z_i,f(z_i))$ and $\fhat_w = \arginf_{f\in\cF} \Rhat_w(f)$. With these notations, a natural empirical counterpart of the population objective~\eqref{eq:minimax-pop-w} can be written as

\begin{equation}
    \fhat_{\alg} := \arginf_{f \in \cF}\sup_{w \in \cW} \big[\Rhat_w(f) - \inf_{f'\in \cF} \Rhat_w(f') \big]= \arginf_{f \in \cF}\sup_{w \in \cW} \big[\Rhat_w(f) - \Rhat_w(\fhat_w)\big],
    \label{eq:minimax-emp-w}
\end{equation}

We now give a couple of concrete examples to instantiate this general formulation.% which will be revisited throughout our presentation. 

\begin{example}[Covariate shift in supervised learning]
Suppose the samples $z = (x,y)$ where $x\in \R^d$ are features and $y\in\R$ is a prediction target. Suppose that the class $\cW$ contains functions $\cW = \{w(z) = w_\theta(x)~:~\theta\in\Theta\}$, where $\Theta$ is some parameter class. That is, we allow only the marginal distribution of $x$ to change while the conditional $\dref(y|x)$ is identical across test distributions~\citep{shimodaira2000improving,huang2006correcting}. In our formulation, it is easy to incorporate the covariate shift situation with a general class $\cW$ under the  restrictions $w_\theta(x) \leq B$. It is possible to add additional regularity on the class $\Theta$, such as taking the unit ball in an RKHS~\citep{huang2006correcting,staib2019distributionally}.
%as we will discuss in the sequel.  
We note that most prior studies of \dro imposed conditions 
on the joint perturbations of both $(x,y)$, which leads to closed form solutions. 
However, directly instantiating $\cW$ in \dro  with a bounded $f$-divergence or MMD perturbations for only the marginal $\dref(x)$
(instead of the joint distribution of $(x,y)$) does not yield a nice closed form solution.  This results in additional computational and statistical challenges as discussed in~\citet{duchi2019distributionally}.
\label{ex:covar-shift}
\end{example}

We now demonstrate the versatility of our setting with an example from reinforcement learning.

\begin{example}[Offline contextual bandit learning]
In offline contextual bandit (CB) learning, the data consists of $z = (x,a,r)$ with $x\in\R^d$ denoting a context, $a\in[K]$ being an action out of $K$ possible choices and $r\in[0,1]$ a reward. There is a fixed and unknown joint distribution $\dref$ over $(x,r(1),\ldots,r(K))$. The object of interest is a decision policy $\pi$ which maps a context $x$ to a distribution over actions, and we seek a policy which maximizes the expected reward under its action choices: $\pi^\star = \argmax_{\pi\in\Pi} \rE[r(a)~|~a\sim\pi(\cdot|x),x]$, where $\Pi$ is some given policy class. In the offline setting, there is some fixed policy $\mu$ which is used to choose actions during the data collection process. Then the training data can be described as $z = (x,a,r)$ where $x\sim \dref$, $a\sim \mu(\cdot | x)$ and $r\sim \dref(\cdot | x,a)$, and this differs from the action distributions that other policies in $\Pi$ induce. Existing literature on this problem typically creates an estimator $\eta$ for $\rE[r | x,a]$ and outputs $\pi = \argmax_{\pi\in\Pi} \sum_{i=1}^n \sum_a \pi(a | x_i)\eta(r|a,x_i)$, and the quality of the resulting policy critically relies on the quality of the estimator $\eta$. A particularly popular choice is the class of \emph{doubly robust estimators}~\citep{cassel1976some,dudik2014doubly} which solve a (weighted) regression problem to estimate $\eta \approx \argmin_f \sum_i \tfrac{\pi(a_i|x_i)}{\mu(a_i|x_i)}(f(x_i,a_i) - r_i)^2$. Since we require the reward model to predict well on the actions selected according to different policies $\pi\in\Pi$, the ideal regression objective for any policy $\pi$ reweights the data as per that policy's actions, and doing this weighting results in a significant performance gain~\citep{su2020doubly,farajtabar2018more}. Typically this weighting is simplified or skipped during policy learning, however, as one needs to simultaneously reweight for all policies $\pi\in\Pi$, and \alg provides an approach to do exactly this by choosing $\cW = \{w(x,a,r) = \pi(a\mid x)/\mu(a\mid x)~:~\pi\in\Pi\}$ in the optimization for estimating $\eta$.% from the offline dataset.
\end{example}

\paragraph{Boundedness and complexity assumptions} We now make standard technical assumptions on $\cF$, the loss function $\ell$ and the weight class $\cW$. To focus attention on the essence of the results, we omit relatively complex results such as chaining or local Rademacher complexity, and use an $\ell_\infty$- covering for uniform convergence over $\cF$, where we use $N(\epsilon, \cF)$ to denote the $\ell_\infty$- covering number of $\cF$ with an accuracy $\epsilon$. For parametric classes, such an analysis is optimal up to a log factors. We also make a boundedness and Lipschitz continuity assumption on the loss function.

\begin{assumption}[Bounded and Lipschitz losses]
The loss function $\ell(z,v)$ is bounded in $[0,1]$ for all $z\in\cZ$ and $v \in \{f(z)~:~f\in\cF, z\in\cZ\}$ and is $L$-Lipschitz continuous with respect to $v$.
\label{ass:loss-bdd}
\end{assumption}

We also assume that the weight class satisfies a boundedness condition.

\begin{assumption}[Bounded importance weights]
All weights $w\in\cW$ satisfy \mbox{$w(z)\leq B_w\leq B$} for all $z\in\cZ$ and some constant $B \geq 1$.
\label{ass:weight-bdd}
\end{assumption}

To avoid the complexity of introducing an additional covering number, throughout this paper, we assume that the weight class $\cW$ is of finite cardinality, which is reasonable for many domain adaptation applications. We adopt the notations $d_{\cF}(\delta) = 1+ \log\frac{N(1/(nLB),\cF)}{\delta}$ and $d_{\cF,\cW}(\delta) = d_{\cF}(\delta) + \ln(|\cW|/\delta)$ to jointly capture the complexities of $\cF$ and $\cW$, given a failure probability $\delta$.

\section{Regret Bounds for Bounded, Lipschitz Losses}
\label{sec:theory-slow}

We begin with the most general case with fairly standard assumptions on the loss function, and the function and weight classes and show the following result on the regret of $\fhat_{\alg}$ for any $w\in\cW$.

\begin{theorem}
    Under assumptions~\ref{ass:loss-bdd} and~\ref{ass:weight-bdd}, suppose further that $\rE_{\dref}[w^2] \leq \sigma_w^2$ for any $w\in\cW$. Then with probability at least $1-\delta$, we have $\forall w\in\cW$:
    \begin{align*}
        \regret_w(\fhat_{\alg}) \leq \inf_{f\in\cF} \sup_{w'\in\cW} \regret_{w'}(f)  + \sup_{w' \in \cW} \underbrace{\order\bigg(\sqrt{\frac{\sigma_{w'}^2 d_{\cF,\cW}(\delta)}{n}} + \frac{B_{w'} d_{\cF,\cW}(\delta)}{n}\bigg)}_{\epsilon_{w'}}.
    \end{align*}
    \label{thm:slow}
\end{theorem}

The bound of Theorem~\ref{thm:slow} highlights the main difference of our objective compared with DRO style approaches. The bound states that we our solution $\fhat_{\alg}$ has a small regret \emph{for each $w\in\cW$}, as long as at least one such function exists in the function class, up to the usual finite sample deviation terms. Thus, $\fhat_{\alg}$ attains the uniform regret guarantee we asked for. To better interpret our result, we state an easy corollary, before making some additional remarks.

\begin{corollary}
Under conditions of Theorem~\ref{thm:slow}, suppose further that $\exists f^\star \in \cF$ such that $R_w(f^\star) \leq R_w(f_w) + \tilde{\epsilon}_w$ for all $w\in\cW$. Then with probability at least $1-\delta$, we have
\begin{align*}
    \forall w\in \cW~:~R_w(\fhat_{\alg}) \leq R_w(f^\star) + \underbrace{\sup_{w'\in\cW} \tilde{\epsilon}_{w'} + \sup_{w'\in\cW} \epsilon_{w'}}_{\epsilon_\cW}.
\end{align*}
\label{cor:slow-samemin}
\end{corollary}

\paragraph{Comparison with DRO approaches.} While Proposition~\ref{prop:dro-noise} already illustrates the improved robustness of \alg to differing noise levels across target distributions, it is further instructive to compare the guarantees for the two approaches. Concretely, in our setting, it can be shown that with probability at least $1-\delta$, we have for all $w\in\cW$:
\begin{align}
    R_w(\fhat_{\dro}) \leq \inf_{f\in\cF}\sup_{w'\in\cW} R_{w'}(f) + \sup_{w' \in \cW} \epsilon_{w'}.
    \label{eq:dro-slow}
\end{align}
Compared with Theorem~\ref{thm:slow}, this bound can be inferior when the different distributions identified by $w\in\cW$ have vastly different noise levels. For instance, in the setting of Corollary~\ref{cor:slow-samemin}, the bound~\eqref{eq:dro-slow} yields $R_w(f_{\dro}) \leq \sup_{w\in\cW} R_w(f^\star) + \epsilon_\cW$, which has additional dependence on the worst case risk of $f^\star$ across all the $w\in\cW$. We can now recall that the assumption of Corollary~\ref{cor:slow-samemin} is satisfied with small $\epsilon_\cW$ in the constructed example of Proposition~\ref{prop:dro-noise}, in which \dro tries to fit to the large noise. Similarly, if the data distribution is Gaussian with an identical mean but different variances across the $w$, then we see that the two bounds reduce to:
\begin{align*}
\textstyle    R_w(\fhat_{\alg}) \leq \text{Var}_w +  \epsilon_{\cW}\quad\mbox{and}\quad R_w(\fhat_{\dro}) \leq \sup_{w'\in\cW} \text{Var}_{w'} +   \epsilon_{\cW},
\end{align*}
where $\text{Var}_w$ is the variance of the Gaussian corresponding to importance weights $w$. Overall, these examples serve to illustrate that the \dro objective is reasonable when the different distributions have similar noise levels, but is not robust to heterogeneity in noise levels. 

\paragraph{Dependence on the complexity of $\cW$.} For discrete $\cW$, our results incur a union bound penalty of $\ln|\cW|$. We suspect that this dependency is unavoidable in the general case.
If we define $\cW$ using a bound constraint such as $w(z) \leq B$, or an $f$-divergence based constraint as in the DRO literature, then the worst-case $w$  has a simple solution in the \dro setting, with no explicit dependence on the complexity of $\cW$. These results are easily adapted to \alg as well by observing  
%\begin{align*}
    $\sup_{w\in\cW} \regret_w(f) = \sup_{w\in\cW}\sup_{f'\in\cF} R_w(f) - R_w(f')$.
%\end{align*} 
For instance, if $\cW_B = \{w~:~0 \leq w(z) \leq B, \rE_{\dref}[w] = 1\}$ is the class of all bounded importance weights, then the results of~\citet{shapiro2017distributionally} adapted to \alg imply that
\begin{align}
\textstyle   \sup_{w\in\cW_B} \regret(f) = \sup_{f'\in\cF}\inf_{\eta\in\R} \left\{\eta + B \rE_{\dref}\left[\left(\ell(f) - \ell(f') -\eta\right)_+\right]\right\}. 
    \label{eq:mro-bdd}
\end{align} 
Similar arguments hold for the family of Cressie-Read divergences studied in~\citet{duchi2021learning}. However, these closed form maximizations rely on doing joint perturbations over $(x,y)$ in the supervised learning setting and do not apply to covariate shifts~\citep{duchi2019distributionally}.

\begin{proof}[Sketch of the proof of Theorem~\ref{thm:slow}]
    For a fixed $f$, we can bound $|\Rhat_w(f) - R_w(f)|$ using Bernstein's inequality. This requires bounding the range and variance of the random variable $w(z)\ell(z,f(z))$. Since the losses are bounded in $[0,1]$, the two quantities are bounded by $\sigma_w^2$ and $B_w$ respectively. Assumption~\ref{ass:loss-bdd} allows us to further get a uniform bound over $f\in\cF$ with an additional dependence on $d_{\cF}(\delta)$. Standard arguments now yield closeness of $\Rhat_w(f) - \Rhat_w(\fhat_w)$ to $R_w(f) - R_w(f_w)$ simultaneously for all $f\in\cF$ and $w\in\cW$, and utilizing the definition of $\fhat_{\alg}$  as the minimizer of $\sup_{w\in\cW}\Rhat_w(f) - \Rhat_w(\fhat_w)$ completes the proof.
\end{proof}

\section{Fast Rates for Squared Loss and Convex Classes}
\label{sec:theory-fast}

A limitation of Theorem~\ref{thm:slow} is that it does not leverage any structure of the loss function beyond Assumption~\ref{ass:loss-bdd}, and as a result obtains a $1/\sqrt{n}$ dependence on the number of samples. For the case of a fixed distribution, it is well known that self-bounding properties of the loss variance can be used to obtain faster $1/n$ bounds on the regret of the ERM solution, and here we show that this improvement extends to our setting. For ease of exposition, we specialize to the squared loss setting in this section. That is, our samples $z$ take the form $(x,y)$ with $x\in\cX\subseteq \R^d$, $y\in[-1,1]$ , our functions $f$ map from $\cX$ to $[-1,1]$ and and $\ell(z,f(z)) = (f(x) - y)^2$. We make an additional convexity assumption on the function class $\cF$.

\begin{assumption}
The class $\cF$ is convex: $\forall~f,f'\in\cF$, $\alpha f+(1-\alpha)f'\in\cF$ for all $\alpha \in [0,1]$. 
\label{ass:convex-class}
\end{assumption}
Note that convexity of $\cF$ is quite different from convexity of $f$ in its parameters, and can always be satisfied by taking the convex hull of a base class. The assumption can be avoided by replacing our ERM based solution with aggregation approaches~\citep[see e.g.][]{tsybakov2003optimal,dalalyan2012sharp}. However, we focus on the convex case here  to illustrate the key ideas.

\begin{theorem}
    Under assumptions~\ref{ass:loss-bdd}, \ref{ass:weight-bdd} and~\ref{ass:convex-class}, with probability at least $1-\delta$, we have
    $\forall w\in\cW$:
    \begin{align*}
         \regret_w(\fhat_{\alg}) \leq &
           \regret_\star   + \order\Big(
        \sqrt{ \regret_\star \cdot B d_{\cF,\cW}(\delta)/n} +
         B d_{\cF,\cW}(\delta)/n\Big)\\ = &\order\Big(\regret_\star + \frac{B d_{\cF,\cW}(\delta)}{n}\Big) , ~~\mbox{where $\regret_\star=\inf_{f \in \cF} \sup_{w' \in \cW} \regret_{w'}(f) $.}
    \end{align*}
    
    \label{thm:fast}
\end{theorem}

Theorem~\ref{thm:fast} crucially leverages the following property of convex classes $\cF$ and squared loss.

\begin{lemma}
    For a convex class $\cF$ and squared loss $\ell(y,f(x)) = (y - f(x))^2$, we have for any distribution $P$ and $f\in\cF$:
%    \[
        $\rE_P[(f(x) - f_P(x))^2] \leq \regret_P(f)$.
%    \]
    \label{lemma:convex-var}
\end{lemma}

Note that a similar result holds for the squared loss and any class $\cF$, if we replace $f_P$ with the unconstrained minimizer of $R_P(f)$ over all measurable functions, but using this property in our analysis would result in an additional approximation error term in the bound of Theorem~\ref{thm:fast}, which is undesirable when considering regret within the function class.

\begin{proof}[Sketch of the proof of Theorem~\ref{thm:fast}]
The only difference with the proof of Theorem~\ref{thm:slow} is in the handling of the variance of $A = w(z)(\ell(z,f(z)) - \ell(z,f_w(z))$, which we use as our random variable of interest in this case. Since $w(z) \leq B_w$ almost surely, we get 
\begin{align*}
    \rE_{\dref}[A^2] &\stackrel{(a)}{\leq} 16B_w\rE_w[(f(x) - f_w(x))^2] 
    \leq 16B_w\regret_w(f),
\end{align*}
where the inequality $(a)$ follows from the boundedness of $f\in\cF$ and $y$ and the final inequality uses Lemma~\ref{lemma:convex-var}. We can now follow the usual recipe for fast rates with a self-bounding property of the variance. Rest of the arguments mirror the proof of Theorem~\ref{thm:slow}.
\end{proof}
\iffalse
Similar to Corollary~\ref{cor:slow-samemin}, if $\exists f^\star \in \cF$ such that $R_w(f^\star) \leq R_w(f_w) + \order\left(\frac{B d_{\cF,\cW}(\delta)}{n}\right)$, then the bound of Theorem~\ref{thm:fast} improves to 
\begin{equation*}
    \forall w\in\cW~:~R_{w}(\fhat_{\alg}) \leq R_w(f^\star) +  \order\left(\frac{B d_{\cF,\cW}(\delta)}{n}\right).
\end{equation*}
\fi
\paragraph{Comparison with \dro.} It is natural to ask if it is possible to obtain fast rate for \dro  under the conditions of Theorem~\ref{thm:fast}. Of course it is not possible to show a regret bound similar to Theorem~\ref{thm:fast} for \dro, since it does not optimize the worst-case regret as Proposition~\ref{prop:dro-noise} illustrates. Nevertheless, we investigate if \dro can achieve fast rate for the raw risk. Unfortunately, the following result shows that even under the assumptions of Theorem~\ref{thm:fast}, the worst-case risk criterion for \dro is still subject to the slower $1/\sqrt{n}$ rate. The reason is that regret has a markedly smaller variance than the risk, and the latter usually deviates at a $1/\sqrt{n}$ rate even for squared loss.
\begin{proposition}
Under the assumptions of Theorem~\ref{thm:fast}, there is a family of distributions $\cW$ with $|\cW| = 2$, satisfying $B_w \leq B = 2$, and a function class $\cF$ with $d_{\cF} = \order(\ln(nL))$, such that with probability at least $1-1/n$, we have 
%\begin{align*}
    $\sup_{w\in\cW} R_w(\fhat_{\dro}) - \sup_{w\in\cW} R_w(f_{\dro}) = \Omega(\sqrt{(\ln n)/n})$.
%\end{align*}
\label{prop:dro-slow}
\end{proposition}

%The result is not surprising since \dro minimizes a significantly noisier estimate $\sup_{w\in\cW} \Rhat_w(f)$ for $\sup_{w\in\cW} R_w(f)$, than \alg does for the corresponding regret, causing a discrepancy between the two methods.

%\alekh{Probably omit this paragraph}
\iffalse
\paragraph{Scaling with importance weights.} While Theorem~\ref{thm:fast} improves the dependence on sample size, it does not leverage the distributional properties of $w(z)$ beyond an $\ell_\infty$ bound. Of course, the final guarantee is always better than that of Theorem~\ref{thm:slow} as the $1/n$ term is identical in the two results. For problems where the bound $B_w$ is too large, additional truncation arguments in the importance weighting literature~\citep{Bembom2008DataadaptiveSO,bottou2013counterfactual,su2020doubly} can be easily composed with our analysis to obtain finer bias-variance tradeoffs.
\fi
\section{Adapting to Misspecification through Non-uniform Scaling across Distributions}
\label{sec:scaling}

Our result so far guarantees uniform regret, with learning complexity that depends on the worst case complexity over all distributions $w$. For example, if the complexity of the distribution family is characterized by $B_w$ as in our analysis, then the bound depends on 
$\sup_w B_w$. In this section, we show that it is possible to improve such dependence using an adaptive scaling of the objective function. 
\iffalse
Given our results so far, an astute reader might notice that different distributions corresponding to different weights $w$ do not incur the same amount of finite sample deviation in our risk bounds, but they are treated homogenously by our estimator that compares the empirical regret across the different $w$'s at the same scale. 
\fi
To better illustrate the consequences of this approach, let us assume that we can rewrite our guarantees in the form that for each $w\in\cW$ and $f\in\cF$, we have with probability at least $1-\delta$, 
\begin{equation}
    \regret_w(f) \leq c (\Rhat_w(f) -  \Rhat_w(\fhat_w)) + c_w\epsilon,
    \label{eq:risk-scale}
\end{equation}
where the scaling function $c_w$ is a distribution $P$ dependent quantity that is known to the algorithm. For instance, under conditions of Theorem~\ref{thm:slow}, we get this bound with $c=1$, $c_w = \sigma_w + B_w/\sqrt{n}$ and $\epsilon = d_{\cF,\cW}(\delta)/\sqrt{n}$ (assuming $d_{\cF,\cW}(\delta) \geq 1$). Under conditions of Theorem~\ref{thm:fast}, we can use $c=3$, $c_w = B_w$ and $\epsilon = d_{\cF,\cW}/n$, by taking the second $O(B/n)$ bound of Theorem~\ref{thm:fast}. While this does worsen our dependence on $d_{\cF,\cW}(\delta)$ in the first case, this is preferable to assuming that $d_{\cF,\cW}(\delta)$ is known and setting of $c_w = \sigma_w + B_w\sqrt{d_{\cF,\cW}(\delta)/n}$ and $\epsilon = \sqrt{d_{\cF,\cW}(\delta)/n}$.% is possible. 

Since we assume $c_w$ is known, then we can define the estimator (Scaled \alg, or \algscale):

\begin{equation}
    \fhat_{\algscale} := \arginf_{f \in \cF}\sup_{w \in \cW} (\Rhat_w(f) - \Rhat_w(\fhat_w))/c_w,
    \label{eq:minimax-emp-w-scale}
\end{equation}

We now present our main results for the statistical properties of $\fhat_{\algscale}$ under the assumptions of Theorems~\ref{thm:slow} and \ref{thm:fast}, before comparing the results and discussing some key consequences.

\begin{theorem}
    Under assumptions of Theorem~\ref{thm:slow}, consider the estimator $\fhat_{\algscale}$ with $c_w = \sigmah_w + B_w/\sqrt{n}$, where $\sigmah_w^2 = \frac{1}{n}\sum_{i=1}^n w(z_i)^2$ is the empirical second moment of the importance weights. Then with probability at least $1-\delta$, we have $\forall w\in\cW$:
    \begin{align*}
        \regret_w(\fhat_{\algscale}) \leq  c_w\inf_{f\in\cF} \sup_{w'\in\cW} \frac{\regret_{w'}(f)}{c_{w'}} + \underbrace{d_{\cF,\cW}(\delta)\;\order\left(\sqrt{\sigma_w^2/n} + B_w/n\right)}_{\epsilon'_w}.
    \end{align*}
    \label{thm:slow-scale}
\end{theorem}
We can also state a version for squared loss and convex classes.

\begin{theorem}
    Under assumptions of Theorem~\ref{thm:fast}, consider the estimator $\fhat_{\algscale}$ with $c_w = B_w$. Then with probability at least $1-\delta$, we have
    \begin{align*}
        \forall w\in\cW~:~ \regret_w(\fhat_{\algscale}) \leq  B_w
        \left[ \regret_\star + \order\left(\sqrt{\regret_\star \cdot \frac{ d_{\cF,\cW}(\delta)}{n}} + \frac{ d_{\cF,\cW}(\delta)}{n}\right)
        \right] ,
    \end{align*}
    where $\regret_\star=\inf_{f\in\cF} \sup_{w'\in\cW} \frac{\regret_{w'}(f)}{B_{w'}}$. 
    \label{thm:fast-scale}
\end{theorem}

Comparing Theorems~\ref{thm:slow-scale} and~\ref{thm:fast-scale} with their counterparts of \alg, we notice a subtle but important difference. The finite sample deviation term $\epsilon'_w$ in Theorem~\ref{thm:slow-scale} depends on the weights $w$ for which the bound is stated, while the corresponding term in Theorem~\ref{thm:slow-scale} is $\sup_{w'\in\cW} \epsilon_{w'}$. In other words, the heterogeneous scaling in \algscale allows us to obtain a regret bound for each distribution depending on the deviation properties of that particular distribution, unlike in the basic \alg approach. To see the benefits of this approach, we state a corollary of Theorems~\ref{thm:slow-scale} and~\ref{thm:fast-scale} next.

\begin{corollary}[Aligned distribution class]
Suppose that $\dref$ and $\cW$ satisfy that $R_w(f_w) = R_w(f^\star)$ for all $w\in\cW$, with $f^\star \in \cF$. Under conditions of Theorem~\ref{thm:slow-scale}, with probability at least $1-\delta$:
    \begin{align*}
        \forall w\in\cW~:~R_w(\fhat_{\algscale}) \leq R_w(f^\star) + d_{\cF,\cW}(\delta)\;\order\left(\sqrt{\frac{\sigma_w^2}{n}} + \frac{B_w}{n}\right).
    \end{align*}
In the same setting, under the assumptions of Theorem~\ref{thm:fast-scale}, we have with probability at least $1-\delta$:
    \begin{align*}
        \forall w\in\cW~:~ R_w(\fhat_{\algscale}) \leq R_w(f^\star) + \order\left(\frac{B_{w} d_{\cF,\cW}(\delta)}{n}\right).
    \end{align*}
    \label{cor:well}
\end{corollary}

Both results follow by choosing $f = f_w = f^\star$ in Theorems~\ref{thm:slow-scale} and~\ref{thm:fast-scale}. This result shows that the rescaling allows our bounds to adapt to the closeness of a target distribution to the data collection distribution, simultaneously for all target distributions. The alignment condition of a shared $f^\star$ is always true in well-specified problems, and we illustrate a consequence of Corollary~\ref{cor:well} for well-specified linear regression in Appendix~\ref{sec:linear}.% below.

\section{Algorithmic considerations}
\label{sec:algo}

So far our development has focused on the statistical properties of \alg. In this section, we discuss how the \alg estimator can be computed from a finite dataset, given some reasonable computational assumptions on the function class $\cF$.

\begin{definition}[ERM oracle]
Given a dataset of the form $(\omega_i, z_i)_{i=1}^n$, an ERM oracle for $\cF$ solves the weighted empirical risk minimization problem: $\min_{f\in\cF} \omega_i \ell(z_i, f(z_i))$.
\label{def:erm}
\end{definition}
While we assume access to an exact ERM oracle, we can weaken the notion to an approximate oracle with an optimization error comparable to the statistical error from finite samples. Given such an oracle, we can now approximately solve the \alg objective~\eqref{eq:minimax-emp-w} (or~\eqref{eq:minimax-emp-w-scale}) by using a well-known strategy of solving minimax problems as two player zero-sum games. To do so, we recall our assumption of a finite class $\cW$, and let $\Delta(\cW)$ denote the set of all distributions over the importance weight family. Similarly $\Delta(\cF)$ is the set of distributions over our function class. Then we can rewrite the objective in Equation~\eqref{eq:minimax-emp-w} as:

\begin{equation}
    \inf_{f \in \cF}\sup_{w \in \cW} \Rhat_w(f) - \Rhat_w(\fhat_w) = \inf_{P \in \Delta(\cF)}\sup_{\rho\in\Delta(\cW)} \rE_{f\sim P, w\sim\rho} \left[\Rhat_w(f) - \Rhat_w(\fhat_w)\right].
    \label{eq:minimax-emp-w-dist}
\end{equation}

The objective is bilinear in $P$ and $\rho$ so that the following result holds due to the exchangeability of minimum and maximum for convex-concave saddle point problems. 
It is also a direct consequence of Proposition~\ref{prop:algo}.
\begin{proposition}
There exists $\hat{\rho} \in \Delta(\cW)$ so that the solution of \alg is equivalent to the following solution of the weighted ERM method:
%\[
$\fhat_{\alg}=\argmin_{f \in \cF} 
\rE_{w \sim\hat{\rho}} \Rhat_w(f)$.
%\]
\end{proposition}

For finite $\cW$ which we consider in this paper, it is also possible to find the approximate $\hat{\rho}$ and 
$\fhat_{\alg}$ efficiently, using the celebrated result of~\citet{freund1996game} to solve this minimax problem through no-regret dynamics. We use the best response strategy for the $P$-player, as finding the best response distribution given a specific $\rho$ is equivalent to finding the function $f\in\cF$ that minimizes $\rE_{w\sim\rho} \Rhat_w(f)$, which can be accomplished through one call to the ERM oracle. For optimizing $\rho$, we use the exponentiated gradient algorithm of~\citet{kivinen1997exponentiated} (closely related to Hedge~\citep{freund1997decision}), which is a common no-regret strategy to learn a distribution over a finite collection of experts, with each $w\in\cW$ being an "expert" in our setting. More formally, we initialize $\rho_1$ to be the uniform distribution over $\cW$ and repeatedly update:
\begin{align}
    f_t &= \arginf_{f\in\cF} \rE_{w\sim\rho_t} \Rhat_w(f),\quad \rho_{t+1}(w) \propto \rho_t(w)\exp\left(\eta \bigg(\Rhat_w(f_t) - \Rhat_w(\fhat_w)\bigg)\right),
    \label{eq:algo}
\end{align}
%
%where we normalize the empirical regret by $B$ in order to match the assumptions of~\citet{shwartz2012online} who assume that the objective has entries bounded in $[-1,1]$. 
Let us denote the distribution $P_t = (f_1+\ldots+f_t)/t$ for the iterates generated by~\eqref{eq:algo}. Then the results of~\citet{freund1996game} yield the following suboptimality bound on $P_t$. 

\begin{proposition}
    For any $T$ and using $\eta =\sqrt{\tfrac{\ln|\cW|}{B^2T}}$ in the updates~\eqref{eq:algo}, the distribution $P_T$ satisfies
    \begin{equation*}
        \rE_{f\sim P_T} \sup_w [\Rhat_w(f) - \Rhat_w(f_w) ]\leq 2B \sqrt{\frac{\ln|\cW|}{T}}.
    \end{equation*}
    \label{prop:algo}
\end{proposition}
At a high-level, Proposition~\ref{prop:algo} allows us to choose $T = n^2$ to ensure that the optimization error is no larger than the statistical error, allowing the same bounds to apply up to constants. The optimization strategy used here bears some resemblance to boosting techniques, which also seek to reweight the data in order to ensure a uniformly good performance on all samples, though the reweighting is not constrained to a particular class $\cW$ unlike here. Note that while the optimization error scales logarithmically in $|\cW|$, the computational complexity is linear in the size of this set, meaning that our strategy is computationally feasible for sets $\cW$ of a modest size. On the other hand, our earlier discussion~\eqref{eq:mro-bdd} suggests that alternative reformulations of the objective might be computationally preferred in the case of class $\cW$ which are continuous. Handling the intermediate regime of a large discrete class $\cW$ in a computationally efficient manner is an interesting question for future research.

\section{Conclusion}
In this paper, we introduce the \alg criterion for problems with distribution shift, and establish learning theoretic properties of optimizing the criterion from samples. We demonstrate the many benefits of reasoning with \emph{uniform regret} as opposed to uniform risk guarantees, and we expect these observations to have implications beyond the setting of distribution shift.

On a technical side, it remains interesting to further develop scalable algorithms for large datasets and weight classes. Better understanding the statistical scaling with the size of the weight class and refining our techniques for important scenarios such as covariate shift are also important directions for future research.

\bibliographystyle{plainnat}
\bibliography{myrefs}

\appendix

\section{Proof of Theorem~\ref{thm:slow}}

Recalling our assumptions that for any $w\in\cW$, we have $\max_z w(z) \leq B$ and $\rE_{z\sim \dref} w(z)^2 \leq \sigma_w^2$, we know that for a fixed $w \in \cW$ and $f\in\cF$, we have, with probability at least $1-\delta$:
\begin{equation}
    \left|\Rhat_{w}(f) - R_w(f)\right| = \order\left(\sqrt{\frac{\sigma_w^2\ln(1/\delta)}{n}} + \frac{B_w\ln(1/\delta)}{n}\right).
    \label{eq:f-dev}
\end{equation}
This is a consequence of Bernstein's inequality applied to the random variable $A = w(z)\ell(z,f(z))$ which is bounded by $B_w$ almost surely, and has a second moment at most $\sigma_w^2$ when $z\sim \dref$. Since $\ell(z,f(z))$ is $L$-Lipschitz in the second argument, if $\|f - f'\|_\infty \leq 1/(nLB)$, we have for any $w\in\cW$ and any $z$:
\begin{align*}
    |w(z)\ell(z,f(z)) - w(z)\ell(z,f'(z))| \leq \frac{B_w L}{nLB} \leq \frac{1}{n},
\end{align*}
where the second inequality follows since $B_w \leq B$ for all $w\in\cW$ by assumption. Hence, defining $\cF'$ to be an $\ell_\infty$ cover of $\cF$ at a scale $1/nLB$, a union bound combined with the bound~\eqref{eq:f-dev} yields that with probability $1-\delta$, we have for all $f\in\cF'$:
\begin{equation*}
    \left|\Rhat_{w}(f) - R_w(f)\right| = \order\left(\sqrt{\frac{\sigma_w^2\ln(N(\cF,1/(nLB))/\delta)}{n}} + \frac{B_w\ln(N(\cF,1/(nLB))/\delta)}{n}\right).
\end{equation*}
Using the Lipschitz property of the loss and further taking a union bound over $w\in\cW$, this gives for all $f\in\cF$ and $w\in\cW$, with probability at least $1-\delta$:
\begin{align}
    \left|\Rhat_{w}(f) - R_w(f)\right| &= \order\left(\sqrt{\frac{\sigma_w^2\ln(N(\cF,1/(nLB))|\cW|/\delta)}{n}} + \frac{B_w\ln(N(\cF,1/(nLB))|\cW|/\delta)}{n}\right) + \frac{1}{n} \nonumber\\
    &= \order\left(\sqrt{\frac{\sigma_w^2d_{\cF,\cW}(\delta)}{n}} + \frac{B_w d_{\cF,\cW}(\delta)}{n}\right) =: \epsilon_w,
    \label{eq:slow-dev}
\end{align}
where the second equality recalls our definition of $d_{\cF,\cW} = 1+\ln\frac{N(1/(nLB),\cF)}{\delta} + \ln\frac{|\cW|}{\delta}$. We now condition on the $1-\delta$ probability event that the bound of Equation~\ref{eq:slow-dev} holds to avoid stating error probabilities in each bound. In particular, applying the bound above to the empirical minimizer $\fhat_w$ for any $w\in\cW$, we observe that 

\begin{align*}
    R_w(\fhat_w) \leq \Rhat_w(\fhat_w) + \epsilon_w \leq \Rhat_w(f_w) + \epsilon_w \leq R_w(f_w) + 2\epsilon_w.
\end{align*}
Hence, for any $f\in\cF$:
\begin{align*}
    \sup_{w\in\cW}\left[R_w(\fhat_{\alg}) - R_w(f_w)\right] &\leq \sup_{w\in\cW} [R_w(\fhat_{\alg}) - \Rhat_w(\fhat_w)] + \sup_w \epsilon_w \\
    &\leq \sup_{w\in\cW}[\Rhat_w(\fhat_{\alg}) - \Rhat_w(\fhat_w)] + 2\sup_w \epsilon_w\\
    &\leq \sup_{w\in\cW} [\Rhat_w(f) - \Rhat_w(\fhat_w)] + 2\sup_w \epsilon_w \tag{since $\fhat_{\alg} = \arginf_{f\in\cF}\sup_{w\in\cW} \Rhat_w(f) - \Rhat_w(\fhat_w)$}\\
    &\leq \sup_{w\in\cW} [R_w(f) - R_w(\fhat_w)] + 4\sup_w \epsilon_w\\
    &\leq \sup_{w\in\cW} [R_w(f) - R_w(f_w)] + 4\sup_w \epsilon_w.
\end{align*}
Taking an infimum over $f\in\cF$ completes the proof of the theorem.

\section{Proof of Theorem~\ref{thm:fast}}

We first prove Theorem~\ref{thm:fast} using Lemma~\ref{lemma:convex-var}, before proceeding to prove the lemma. 

\begin{proof}[Proof of Theorem~\ref{thm:fast}]
We begin with the deviation bound for a fixed $f\in\cF$ and $w\in\cW$ as before, but use the regret random variable $A = w(z) (\ell(z,f(z)) - \ell(z,f_w(z))$ this time. Then by Lemma~\ref{lemma:convex-var}, we have

\begin{align*}
    \rE_{\dref}[A^2] &\leq B_w \rE_{\dref} [w(z)(\ell(z,f(z)) - \ell(z,f_w(z)))^2] = B_w \rE_{w}[(\ell(z,f(z)) - \ell(z,f_w(z)))^2]\\
    &= B_w \rE_w\left[\left((f(x) - y)^2 - (f_w(x) - y)^2\right)^2\right]\\
    &= B_w \rE_w\left[(f(x) - f_w(x))^2(f(x) + f_w(x) - 2y)^2\right]\\
    &\leq 16B_w \rE_w\left[(f(x) - f_w(x))^2\right] \tag{$(|f(x)|, |f_w(x)|, |y| \leq 1$}\\
    &\leq 16B_w \regret_w(f)   \tag{Lemma~\ref{lemma:convex-var}}.
\end{align*}

Thus we see that for any fixed $f\in\cF$ and $w\in\cW$, we have with probability $1-\delta$:
\begin{align*}
    |\regret_w(f)  - [\Rhat_w(f) - \Rhat_w(f_w)]| &= \order\left(\sqrt{\frac{B_w\regret_w(f)\ln(1/\delta)}{n}} + \frac{B_w\ln(1/\delta)}{n}\right)\\
    &\leq \gamma \regret_w(f) + (1+\gamma^{-1})\order\left(\frac{B_w\ln(1/\delta)}{n}\right).
\end{align*}
where $\gamma>0$ is arbitrary.

Hence, with probability at least $1-\delta$, we have for a fixed $f$ and $w$:
\begin{align*}
    (1-\gamma)\regret_w(f) &\leq (\Rhat_w(f) - \Rhat_w(f_w)) + (1+\gamma^{-1})\order\left(\frac{B_w\ln(1/\delta)}{n}\right), \quad \mbox{and}\\
    \Rhat_w(f) - \Rhat_w(f_w) &\leq  (1+\gamma)\regret_w(f) + (1+\gamma^{-1}) \order\left(\frac{B_w\ln(1/\delta)}{n}\right).
\end{align*}
In particular, choosing $f = \fhat_w$ in the first inequality, we see that

\begin{align}
    \Rhat_w(f_w) \leq \Rhat_w(\fhat_w) + (1+\gamma^{-1})\order\left(\frac{B_w\ln(1/\delta)}{n}\right)
    \label{eq:fw-emp-fast}
\end{align}

Using identical arguments as the proof of Theorem~\ref{thm:slow} allows us to turn both statements into uniform bounds over all $f\in\cF$ and $w\in\cW$. Defining $\epsilon' = B d_{\cF,\cW}(\delta)/n$, we now have with probability at least $1-\delta$:

\begin{align*}
   &(1-\gamma) \sup_{w\in\cW}\regret(\fhat_{\alg})\\
    &\leq \sup_{w\in\cW}[\Rhat_w(\fhat_{\alg}) - \Rhat_w(f_w)] + (1+\gamma^{-1})\epsilon'\\
    &\leq \sup_{w\in\cW}[\Rhat_w(\fhat_{\alg}) - \Rhat_w(\fhat_w)] + (1+\gamma^{-1})\epsilon' \tag{$\Rhat_w(f_w)\geq \Rhat_w(\fhat_w)$}\\ 
    &\leq \sup_{w\in\cW} \Rhat_w(f) - \Rhat_w(\fhat_w) + (1+\gamma^{-1})\epsilon' \tag{since $\fhat_{\alg} = \arginf_{f\in\cF}\sup_{w\in\cW} \Rhat_w(f) - \Rhat_w(\fhat_w)$}\\
    &\leq \sup_{w\in\cW} \Rhat_w(f) - \Rhat_w(f_w) + 2(1+\gamma^{-1})\epsilon' \tag{Equation~\ref{eq:fw-emp-fast}}\\
    &\leq (1+\gamma)\sup_{w\in\cW}  [R_w(f) - R_w(f_w)] + 3(1+\gamma^{-1})\epsilon'\\
    &= (1+\gamma)\sup_{w\in\cW} \regret_w(f) + 3(1+\gamma^{-1})\epsilon'.
\end{align*}

Taking an infimum over $f\in\cF$ and
$\gamma=\min(0.5,\sqrt{\epsilon'/\inf_{f\in\cF}\sup_{w\in\cW}\regret_w(f)})$
completes the proof.
\end{proof}

We now prove Lemma~\ref{lemma:convex-var}.

\subsection{Proof of Lemma~\ref{lemma:convex-var}}

Since $\cF$ is convex, if $f_1 \in \cF$ and $f_2 \in \cF$, then $\alpha f_1 + (1-\alpha) f_2 \in \cF$ for $\alpha \in [0,1]$. Then using $f_P = \arginf_{f\in\cF} R_P(f)$, we have for any distribution $P$ and $f \in \cF$:
\begin{align*}
    0 &\leq R_P(\alpha f + (1-\alpha) f_P) - R_P(f_P)\\
    &= \alpha\rE_P[(\alpha f(x) + (2-\alpha) f_P(x) - 2y)(f(x) - f_P(x))]\\
    &= \alpha^2 \rE_P[(f(x) - f_P(x))^2] + 2\alpha \rE_P[(f_P(x) - y)(f(x) - f_P(x))].
\end{align*}
Since this holds for any $\alpha \in [0,1]$, we take the limit $\alpha \downarrow 0$ to conclude that for all $f \in \cF$
\begin{equation}
    \rE_P[(f_P(x) - y)(f(x) - f_P(x))] \geq 0.
    \label{eq:first-order}
\end{equation}

This inequality further allows us to conclude for any $f\in\cF$
\begin{align*}
    R_P(f) - R_P(f_P) &= \rE_P[(f(x) + f_P(x) - 2y)(f(x) - f_P(x)]\\
    &= \rE_P[(f(x) - f_P(x))^2] + 2\rE_P[(f_P(x) - y)(f(x) - f_P(x))] \\
    &\geq \rE_P[(f(x) - f_P(x))^2]. 
\end{align*}

\subsection{Proof of Proposition~\ref{prop:dro-slow}}

We consider a simple problem where there is a class $\cP = \{P_1,P_2\}$ consisting of two distributions that we want to do well under. Each distribution $P_i$ is such that $x\sim P_i$ takes the form $x = \mu_i + \epsilon$ where $\mu_i \in \R$ and $\epsilon$ is $\pm 1$ with probability $1/2$. The function class $\cF = [-C,C]$ for a constant $C$ that will be appropriately chosen in the proof. Let $\dmu = \mu_1 - \mu_2$ be the difference of means. Let us choose $\dref$ to be supported over $\R\times \{1,2\}$ so that $z = (x,1)$ where $x\sim P_1$ with probability 0.5 and $z = (x,2)$ with $x\sim P_2$ with probability 0.5. We set $w_1(z) = 1$ if $z_2 = 1$ and similarly $w_2(z) = 1$ if $z_2 = 2$. Let us define $g(z) = f(z_1)$ for $f\in\cF$. Since $g(z)$ is equivalent to $f(x)$ with $x = z_1$ in this example, we stick to using the notation $f\in\cF$ for the rest of the proof for consistency with our notation throughout. Let us use $R_i$, $\Rhat_i$ to denote the empirical and expected risks on samples from $P_i$ (equivalently $w_i$). Then a direct calculation shows that

\begin{align*}
    R_i(f) = (f - \mu_i)^2 + 1,\quad \mbox{and} \quad \Rhat_i(f) = (f-\mu_i)^2 + 1 - 2S_i(f-\mu_i),
\end{align*}
where $S_i = \sum_{j=1}^n w_i(z_j)\epsilon_j/(\sum_{j=1}^n w_i(z_j))$ with $\epsilon_j$ being the noise realization in $z_j$. Then for sufficiently large $C$, we see that 
\begin{align*}
    f_{\dro} = \argmin_{f\in\cF} \max_i R_i(f) &= \argmin_{f\in\cF} \max\{(f-\mu_1)^2,(f-\mu_2)^2\} = \frac{\mu_1 + \mu_2}{2},
\end{align*}
and the best worst case risk is given by $\max_i R_i(f_{\dro}) = 1+\dmu^2/4$. Now let us examine the empirical situation. We have

\begin{align*}
    \fhat_{\dro} &= \argmin_{f\in\cF} \max_i \Rhat_i(f) = \argmin_{f\in\cF} \max\{(f-\mu_1)^2 - 2S_1(f-\mu_1),(f-\mu_2)^2 - 2S_2(f-\mu_2)\}.
\end{align*}
Let us denote $\fhat_i = \mu_i + S_i$ as the minimizer of $\Rhat_i(f)$. Let us condition on the event $\event$ where:
\begin{equation}
    \event = \{\Rhat_1(\fhat_1) \leq \Rhat_2(\fhat_1)\quad \mbox{and}\quad \Rhat_1(\fhat_2) \geq \Rhat_2(\fhat_2)\}.
\label{eq:event}
\end{equation}
Under this event, since both $\Rhat_1(f)$ and $\Rhat_2(f)$ are continous functions of $f$, they are equal for some $f \in [\min(\fhat_1, \fhat_2), \max(\fhat_1, \fhat_2)]$. So conditioned on $\event$, we seek a solution to
\begin{align*}
    0 &= (f-\mu_1)^2 - 2S_1(f-\mu_1) - (f-\mu_2)^2 + 2S_2(f-\mu_2)\\
    &= 2f(\mu_2 - \mu_1 + S_2 - S_1) + (\mu_1^2 - \mu_2^2 + 2S_1\mu_1 - 2S_2\mu_2),
\end{align*}
so that we get 
\begin{align*}
\fhat_{\dro} &= \frac{\mu_2^2 - \mu_1^2 + 2S_2\mu_2 - 2S_1\mu_1}{2(\mu_2 - \mu_1 + S_2 - S_1)}\\
&= \frac{\dmu(\mu_1+\mu_2) + 2S_1\mu_1 - 2S_2\mu_2}{2(\dmu + \ds)}\\
&= \frac{\mu_1+\mu_2}{2} + \frac{2S_1\mu_1 - 2S_2\mu_2 -(\mu_1+\mu_2)\ds}{2(\dmu+\ds)}\\
&= \frac{\mu_1+\mu_2}{2} + \frac{S_1\mu_1 - S_2\mu_2 - S_1\mu_2 + S_2\mu_1}{2(\dmu+\ds)}\\
&= \frac{\mu_1+\mu_2}{2} + \underbrace{\frac{(S_1+S_2)\dmu}{2(\dmu + \ds)}}_{\xi}.
\end{align*}

Plugging the minimizer back in, we obtain the best worst case population risk as

\begin{align*}
    \max_i R_i(\fhat_{\dro}) - 1 &= \max\{(\fhat_{\dro} - \mu_1)^2,(\fhat_{\dro} - \mu_2)^2\\
    &= \max\left\{(-\dmu/2 + \xi)^2,(\dmu/2 + \xi)^2 \right\}\\
    &= \dmu^2/4 + \xi^2 + \max\left\{-\dmu\xi, \dmu\xi\right\}\\
    &= \dmu^2/4 + \xi^2 + |\dmu\xi|.
\end{align*}

Recalling that $\min_f\max_i R_i(f) = 1 + \dmu^2/4$, we further obtain
\begin{align*}
    \max_i R_i(\fhat_{\dro}) - \max_i R_i(f_{\dro}) &= \xi^2 + |\dmu\xi|.
\end{align*}

In order to complete the lower bound on this regret term, we now analyze the random variable $\xi$, and we choose $\dmu = 1$. Note that from the definitions of $S_i$ and $n_i$, we have with probability at least $1-4\delta$, for $i=1,2$:
\begin{align*}
    |S_i| &\leq c\sqrt{\frac{\log(1/\delta)}{n_i}}, \quad \mbox{and}\\
    \left|n_i - \frac{n}{2}\right| &\leq c\sqrt{\frac{\log(1/\delta)}{n}},
\end{align*}

Under this $1-4\delta$ event, we have $\ds = \order(\sqrt{\ln(1/\delta)/n})$ and $S_1 + S_2 = O(\sqrt{\ln(1/\delta)/n})$. Consequently, under this event
\begin{align*}
    |\xi| = \order\left(\left|\frac{\dmu\sqrt{\ln(1/\delta)}}{2\sqrt{n}(\dmu - \sqrt{\ln(1/\delta)/n})}\right|\right). 
\end{align*}
Choosing $4\delta = 1/n$, we further have
\begin{align*}
    |\xi| = \order\left(\left|\frac{\dmu\sqrt{\ln n}}{\sqrt{n}\dmu}\right|\right) = \order\left(\sqrt{\frac{\ln n}{n}}\right),
\end{align*}
where we have used $\dmu/2 \geq \sqrt{\ln n/n}$ for $n \geq 1$. Note that we have also not set the constant $C$ defining $\cF$ so far. In order for the unconstrained minimizations above to happen for $f\in\cF$, we need $f_{\dro}, \fhat_{\dro} \in \cF$. This means that $C \geq (\mu_1 + \mu_2)/2$ and $C \geq (\mu_2^2 - \mu_1^2 + 2S_2\mu_2 - 2S_1\mu_1)/(2(\mu_2 - \mu_1 + S_2 - S_1))$. The latter is satisfied under our $1-4/n$ event when $C \geq \mu_1 + \mu_2$. 

Finally we recall that we have conditioned so far on the event $\event$ defined in Equation~\ref{eq:event}. We now examine $\P(\event^C)$ under the $1-4/n$ event, and focus on the first condition in~\eqref{eq:event}. We have

\begin{align*}
    \Rhat_1(\fhat_1) - \Rhat_2(\fhat_2) &= -S_1^2 + S_2^2 - (\dmu + \ds)^2\\
    &\leq \order\left(\frac{\log n}{n} - 1 + \sqrt{\frac{\log n}{n}}\right) \leq 0,
\end{align*}
under the $1-4/n$ probability event. Similar reasoning also shows that $\Rhat_2(\fhat_2) \leq \Rhat_1(\fhat_2)$ in this event, so that we get the conclusion of the proposition.
\section{Proofs for Section~\ref{sec:scaling}}

We start with a technical lemma about the concentration of empirical estimates of the second moment of importance weights, which is required for Theorem~\ref{thm:slow-scale}.

\begin{lemma}
    Under Assumption~\ref{ass:weight-bdd}, we have with probability at least $1-\delta$, for all $w\in\cW$:
    \begin{align*}
        \rE_{\dref} w(z)^2 \leq \frac{2}{n}\sum_{i=1}^n w(z_i)^2  + \order\left(\frac{B_w^2\ln(|\cW|/\delta)}{n}\right).
    \end{align*}
    \label{lem:weight-var}
\end{lemma}

\begin{proof}
    Consider the non-negative random variable $A = w(z)^2$ so that $A \leq B_w^2$ with probability 1 when $z\sim\dref$. Then we have 
    \begin{align*}
        \rE_{\dref} [A^2] &= \rE_{\dref}[w(z)^4] \leq B_w^2 \rE_{\dref}[w(z)^2].
    \end{align*}
    Then for a fixed $w\in\cW$, we have by Bernstein's inequality, with probability at least $1-\delta$:
    \begin{align*}
        &\left|\frac{1}{n}\sum_{i=1}^n w(z_i)^2 - \rE_{\dref} w(z)^2\right|\\ 
        &= \order\left(\sqrt{\frac{B_w^2\rE_{\dref}[w(z)^2]\ln(1/\delta)}{n}} + \frac{B_w^2\ln(1/\delta)}{n}\right)\\
        &\leq \frac{\rE_{\dref}[w(z)^2]}{2} + \order\left(\frac{B_w^2\ln(1/\delta)}{n}\right).
    \end{align*}
    Rearranging terms and taking a union bound over $w\in\cW$ completes the proof. 
\end{proof}

We are now ready to prove Theorem~\ref{thm:slow-scale}.

\subsection{Proof of Theorem~\ref{thm:slow-scale}}

Recall the definition $\sigmah^2_w = \frac{1}{n}\sum_{i=1}^n w(z_i)^2$. Plugging the result of Lemma~\ref{lem:weight-var} in Equation~\ref{eq:f-dev}, we see that with probability at least $1-\delta$, for a fixed $f\in\cF$ and $w\in\cW$ we have:

\begin{align*}
    \left|\Rhat_{w}(f) - R_w(f)\right| &= \order\left(\sqrt{\frac{(2\sigmah_w^2 + B_w^2/n)\ln(2/\delta)}{n}} + \frac{B_w\ln(2/\delta)}{n}\right)\\
    &= \order\left(\sqrt{\frac{\sigmah_w^2\ln(2/\delta)}{n}} + \frac{B_w\ln(2/\delta)}{n}\right).
\end{align*}

Now following the proof of Theorem~\ref{thm:slow}, with probability at least $1-\delta$, we have for all $w\in\cW$:

\begin{align*}
    &R_w(\fhat_{\algscale}) - R_w(f_w)\\ &= \Rhat_w(\fhat_{\algscale}) - \Rhat_w(\fhat_w) + \order\left(\sqrt{\frac{\sigmah_w^2(d_{\cF}(\delta) + \log(|\cW|/\delta)}{n}} + \frac{B_w(d_{\cF}(\delta) + \log(|\cW|/\delta)}{n}\right)\\
    &= \Rhat_w(\fhat_{\algscale}) - \Rhat_w(\fhat_w) + \underbrace{\left(\sigmah_w + \frac{B_w}{\sqrt{n}}\right)}_{c_w}\,\underbrace{\order\left(\frac{d_{\cF,\cW}(\delta)}{\sqrt{n}}\right)}_{\epsilon}.
\end{align*}

Now dividing through by $c_w$, we see that for all $f\in\cF$, we have with probability at least $1-\delta$:

\begin{align*}
    \sup_{w\in\cW}\frac{R_w(\fhat_{\algscale}) - R_w(f_w)}{c_w} &\leq \sup_{w\in\cW}\frac{\Rhat_w(\fhat_{\algscale}) - \Rhat_w(\fhat_w)}{c_w} + \epsilon\\
    &\leq \inf_{f\in\cF}\sup_{w\in\cW}\frac{\Rhat_w(f) - \Rhat_w(\fhat_w)}{c_w} + \epsilon\\
    &\leq \inf_{f\in\cF}\sup_{w\in\cW}\frac{R_w(f) - R_w(\fhat_w)}{c_w} + 2\epsilon\\
    &\leq \inf_{f\in\cF}\sup_{w\in\cW}\frac{R_w(f) - R_w(f_w)}{c_w} + 2\epsilon,
\end{align*}    
where the second inequality follows from the definition of \algscale as the empirical optimizer of the objective~\eqref{eq:minimax-emp-w-scale}. As a result, we have for any $w\in\cW$, with probability at least $1-\delta$:

\begin{align*}
    R_w(\fhat_{\algscale}) - R_w(f_w) &\leq c_w\inf_{f\in\cF}\sup_{w'\in\cW}\frac{R_{w'}(f) - R_{w'}(f_w)}{c_{w'}} +2 c_w\epsilon.
\end{align*}

\subsection{Proof of Theorem~\ref{thm:fast-scale}}

We will be terse as the proof is largely a combination of the proofs of Theorems~\ref{thm:fast} and \ref{thm:slow-scale}. Proceeding as in the proof of Theorem~\ref{thm:fast}, we see that with probability at least $1-\delta$, we have for all $w\in\cW$ and $f \in \cF$:
\[
\regret_w(f) \leq (1+\gamma)
[\Rhat_w(f) - \Rhat_w(\fhat_w)] + (1+\gamma^{-1}) \underbrace{B_w}_{c_w} \, \underbrace{\order\left(\frac{d_{\cF,\cW}}{n}\right)}_{\epsilon}
\]
and
\[
[\Rhat_w(f) - \Rhat_w(\fhat_w)]
\leq (1+\gamma)\regret_w(f)   + (1+\gamma^{-1}) \underbrace{B_w}_{c_w} \, \underbrace{\order\left(\frac{d_{\cF,\cW}}{n}\right)}_{\epsilon}
\]
for all $\gamma >0$.

Dividing through by $c_w$ as before and taking a supremum over $w\in\cW$, we obtain with probability at least $1-\delta$:

\begin{align*}
    \frac{\regret_w(\fhat_{\algscale})}{c_w} &\leq (1+\gamma)\sup_{w\in\cW}\frac{\Rhat_w(\fhat_{\algscale}) - \Rhat_w(\fhat_w)}{c_w} + (1+\gamma^{-1}) \epsilon\\
    &\leq (1+\gamma) \inf_{f\in\cF}\sup_{w\in\cW}\frac{\Rhat_w(f) - \Rhat_w(\fhat_w)}{c_w} + (1+\gamma^{-1})\epsilon\\
    &\leq (1+\gamma)^2\inf_{f\in\cF}\sup_{w\in\cW}\frac{\regret_w(f)}{c_w} + (1+\gamma^{-1})(2+\gamma)\epsilon.
\end{align*}

Now following the remaining proof of Theorem~\ref{thm:slow-scale} gives the desired bound.

\section{Well-specified linear regression with covariate shift}
\label{sec:linear}

Let us consider a special case of Theorem~\ref{thm:fast-scale}, where $\cF=\{\beta^\top x~:~\beta\in\R^d,\|\beta\|_2 \leq 1\}$ is the class of linear prediction functions with unit norm weights and the data satisfies $\|x\|_2 \leq 1$. We further assume that $y = x^\top \beta^\star + \nu$, where $\|\beta^\star\|_2 \leq 1$ and $\nu$ is zero-mean noise such that $|y|\leq C$ for some constant $C$ (this just causes additional scaling with $C$ in the bounds of Theorems~\ref{thm:fast} and~\ref{thm:fast-scale}). Suppose further that the covariance $\Sigma_{\dref} := \rE_{x\sim\dref} xx^\top$ is full rank with the smallest eigenvalue equal to $\lambda$. Let $\betah_{\dref}$ be the ordinary least squares estimator, given samples from $\dref$. Then it is well-known that with probability at least $1-\delta$, 
\begin{align}
    \|\betah_{\dref} - \beta^\star\|_{\Sigma_{\dref}}^2 = \order\left(\frac{d\ln(1/\delta)}{n}\right).
    \label{eq:erm-linear}
\end{align}
Consequently, we have for any other distribution $P$ over $(x,y)$:
\begin{align*}
    R_P(\betah_{\dref}) - R_P(\beta^\star) = \rE_P[(x^\top\betah_{\dref} - x^\top\beta^\star)^2] = \order\left(\frac{d\ln(1/\delta)}{\lambda n}\right),
\end{align*}
where the inequality uses $\|x\|_2\leq 1$ and that the smallest eigenvalue of $\Sigma_{\dref}$ is at least $\lambda$. That is, doing OLS under $\dref$ yields a strong guarantee for all target distributions $P$, since we get pointwise accurate predictions in this scenario. However, directly applying the results of Theorem~\ref{thm:fast} would still have additional scaling with $B_w$ for any target distribution with importance weights $w$. On the other hand, let us consider the bound of Corollary~\ref{cor:well} for any class $\cW$ such that $w_0 \equiv 1$ is in $\cW$, so that we always include $\dref$ in our class of target distributions. Since $B_{w_0} = 1$ and $d_{\cF} = d\ln(\frac1\delta)$ in this case, the second bound of the corollary yields with probability at least $1-\delta$:
\begin{align*}
   \|\betah_{\algscale} - \beta^\star\|_{\Sigma_{\dref}}^2 =  R_{\dref}(\betah_{\algscale}) - R_{\dref}(\beta^\star) =   \order\left(\frac{(d\ln(1/\delta)+\ln\tfrac{|\cW|}{\delta})}{n}\right).
\end{align*}
This is comparable to the prediction error bound in~\eqref{eq:erm-linear} for ERM on $\dref$, only incurring an additional $\ln|\cW|$ term compared. Note that ERM is asymptotically efficient in this setting, so we cannot expect to do better and suffer only a small penalty for our worst-case robustness. The approach of learning on $\dref$ alone is of course not robust to misspecification in the linear regression assumption. The guarantee of Theorem~\ref{thm:fast}, in contrast incurs a bound $\sup_{w\in\cW} B_w \order\left(\frac{(d\ln(1/\delta)+\ln\tfrac{|\cW|}{\delta})}{n}\right)$, which can be significantly worse.

\section{Proofs for Section~\ref{sec:algo}}

It is clearly seen that the updates of $\rho_t$ correspond to the exponentiated gradient (EG) algorithm applied to the linear objective $-(\rE_{w\sim \rho} \Rhat_w(f_t) - \Rhat_w(\fhat_w))/B$ at round $t$, where the negative sign happens since the EG algorithm is designed for minimization problems, while we apply it to a maximization problem. The regret guarantee for EG, specifically Corollary 2.14 of~\citet{shwartz2012online} states that for any $w\in\cW$
\begin{align*}
    -\sum_{t=1}^T \rE_{w'\sim \rho_t} \frac{\Rhat_{w'}(f_t) - \Rhat_{w'}(\fhat_{w'})}{B} \leq -\sum_{t=1}^T \frac{\Rhat_w(f_t) - \Rhat_w(\fhat_w)}{B} + \frac{\ln|\cW|}{\eta B} + \eta B T.
\end{align*}
Multiplying through by $B$ and substituting $\eta B = \sqrt{\ln|\cW|/T}$ gives
\begin{align*}
    -\sum_{t=1}^T \rE_{w'\sim \rho_t} [\Rhat_{w'}(f_t) - \Rhat_{w'}(\fhat_{w'}) ] \leq -\sum_{t=1}^T [\Rhat_w(f_t) - \Rhat_w(\fhat_w)] + 2B\sqrt{T\ln|\cW|}.
\end{align*}

Now recalling the definition $P_t = (f_1+\ldots+f_t)/t$, following the proof technique of~\citet{freund1996game} gives that

\begin{align*}
    & \rE_{f\sim P_T} \sup_w \frac{\Rhat_w(f) - \Rhat_w(f_w)}{B} \\
    &= \frac{1}{T} \sum_{t=1}^T \sup_w \frac{\Rhat_w(f_t) - \Rhat_w(f_w)}{B}\\
    &\leq \frac{1}{T}\sum_{t=1}^T \rE_{w\sim \rho_t} \frac{\Rhat_{w}(f_t) - \Rhat_{w}(\fhat_{w})}{B} + 2\sqrt{\frac{\ln|\cW|}{T}}\\
    &\stackrel{(a)}{\leq} \frac{1}{T} \sum_{t=1}^T \inf_{f\in\cF}\rE_{w\sim \rho_t} \frac{\Rhat_{w}(f) - \Rhat_{w}(\fhat_{w})}{B} + 2\sqrt{\frac{\ln|\cW|}{T}}\\
    &\leq \inf_{f\in\cF}\frac{1}{T} \sum_{t=1}^T\rE_{w\sim \rho_t} \frac{\Rhat_{w}(f) - \Rhat_{w}(\fhat_{w})}{B} + 2\sqrt{\frac{\ln|\cW|}{T}}\\
    &\leq \inf_{f\in\cF} \sup_{w\in\cW}\frac{\Rhat_{w}(f) - \Rhat_{w}(\fhat_{w})}{B} + 2\sqrt{\frac{\ln|\cW|}{T}},
\end{align*}
where $(a)$ follows from the best response property of $f_t$ for $\rho_t$~\eqref{eq:algo}. 

\end{document}